\documentclass[twoside,11pt]{article}

\usepackage{jmlr2e}

\jmlrheading{83}{2018}{}{10/27 2017}{}{}{Ya-Ping Hsieh and Volkan Cevher}

\ShortHeadings{Dimension-free Information Concentration via Exp-Concavity}{Hsieh and Cevher}
\firstpageno{1}

\usepackage[utf8]{inputenc}
\usepackage{enumerate}
\usepackage{amssymb}
\usepackage{amsmath}
\usepackage{graphicx}
\usepackage{xspace}
\usepackage{stmaryrd}
\usepackage{ae}
\usepackage{enumerate}
\usepackage[usenames,dvipsnames,svgnames,table]{xcolor}
\usepackage{bbm}
\usepackage{algorithm, algpseudocode}
\usepackage{mathtools}

\DeclareMathOperator*{\argmax}{arg\,max}

\newcommand{\RR}[0]{\mathbb{R}}
\newcommand{\ip}[2]{\left\langle #1, #2 \right\rangle}

\newcommand{\Var}{\mathrm{Var}}

\def\l({\left(}
\def\r){\right)}
\def\bl({\Big(}
\def\br){\Big)}
\def\beq{\begin{equation}}
\def\eeq{\end{equation}}

\usepackage{tikz,pgfplots}
\def\x{{\mathbf x}}
\def\y{{\mathbf y}}

\def\RR{{\mathbb R}}

\def\a{{\mathbf{a}}}

\def\C2{{\mathcal{C}^2}}

\def\tV{{\tilde{V}}}

\def\calW{{\mathcal{W}}}

\def\PP{{ \mathbb P }}
\def\EE{{ \mathbb E }}

\begin{document}

\title{Dimension-free Information Concentration \\ via Exp-Concavity}

\author{\name Ya-Ping Hsieh \email ya-ping.hsieh@epfl.ch
       \AND
       \name Volkan Cevher \email volkan.cevher@epfl.ch\\
       \addr Laboratory for Information and Inference Systems (LIONS)\\
        \'{E}cole Polytechnique F\'{e}d\'{e}rale de Lausanne (EPFL)\\
       EPFL-STI-IEL-LIONS, Station 11\\
       CH-1015 Lausanne, Switzerland}

\editor{}

\maketitle

\begin{abstract}
Information concentration of probability measures have important implications in learning theory. Recently, it is discovered that the information content of a log-concave distribution concentrates around their differential entropy, albeit with an unpleasant dependence on the ambient dimension. In this work, we prove that if the potentials of the log-concave distribution are \emph{exp-concave}, which is a central notion for fast rates in online and statistical learning, then the concentration of information can be further improved to depend only on the exp-concavity parameter, and hence, it can be dimension independent.  Central to our proof is a novel yet simple application of the variance Brascamp-Lieb inequality. In the context of learning theory, our concentration-of-information result immediately implies high-probability results to many of the previous bounds that only hold in expectation.
\end{abstract}

\begin{keywords}
Dimension-Free Concentration, Log-Concave Measures, Exp-Concavity, Variance Brascamp-Lieb Inequality, Differential Entropy
\end{keywords}



\section{Introduction}
We study the \textbf{information concentration} of probability measures: Given a probability density $f$ and a random variable $X \sim f$, we ask how concentrated is the random variable $-\log f(X)$ around its mean $\EE [-\log f(X)]$, which is simply the differential entropy of $f$. 

We focus on the class of \textbf{log-concave} probability measures, whose densities are of the form $f(\x) \propto e^{-V(\x)}$ for some convex function $V(\cdot)$. Information concentration for log-concave measures has found many applications in learning theory, ranging from aggregation \citep{dalalyan2016exponentially} and Bayesian decision theory \citep{pereyra2017maximum}, to, unsurprisingly, information theory \citep{raginsky2013concentration}. It also has immediate implications to online learning and PAC-Bayesian analysis (\textit{cf.}, Section 4 for further discussion on this topic)

\citet{bobkov2011concentration} discovered the information concentration phenomenon for log-concave measures. Their result was later sharpened by \cite{fradelizi2016optimal}, which establishes the current state-of-the-art. However, via the concentration bound in \citep{fradelizi2016optimal}, one can immediately notice a poor dependence on the dimension (see \textbf{Theorem \ref{thm:concentration_log-concave}}). 

This unpleasant dependence is, however, not due to any deficit of the analysis: Even in the Gaussian case, the information concentration is known to be dimension-dependent \citep{cover1989gaussian}, and the bound in \citep{fradelizi2016optimal} matches the tightest known result. We can verify that the exponential distributions, another candidate for dimension-free concentration, share the same poor dimensional scaling.

Given these observations, one might pessimistically conjecture that no meaningful subclass of log-concave measures satisfies the information concentration in a dimension-free fashion. Hence, our main result comes as a surprise that, not only does there exist a large subclass of log-concave measures with dimension-free information concentration, but in addition, this subclass is extremely well-known to the machine learning community: 

\vspace{14pt}
\begin{minipage}[c]{.9\textwidth}
\textbf{Our main result (informal statement):}  \emph{Let $f(\x) \propto e^{-V(\x)}$ where $V$ is \textbf{$\eta$-exp-concave}. Then, the information concentration of $f(\x)$ solely depends on the exp-concavity parameter $\eta$, and not the ambient dimension.}
\end{minipage}

\vspace{14pt}

Many loss functions in machine learning are known to be exp-concave; a non-exhaustive list includes the squared loss, entropic loss, log-linear loss, SVMs with squared hinge loss, and log loss; see \citep{cesa2006games} for more. Moreover, distributions of the type $e^{-V(\x)}$, where $V$ is exp-concave, appear frequently in many areas of learning theory. Consequently, our main result is tightly connected to learning with exp-concave losses; see Section 4.

Our main insight is that exp-concave functions are Lipschitz in a local norm, and log-concave measures satisfy the ``Poincar\'{e} inequality in this local norm'', namely the Brascamp-Lieb inequality. We elaborate more on the intuition in Section \ref{sec:proof}.1. In retrospect, the proof of our main result is, once the right tools are identified, completely natural and elementary. In fact, our result basically implies that the exp-concavity arises naturally in the dimension-free information concentration. 

The rest of the paper is organized as follows. We first set up notations and review basics of differential entropy and log-concave distributions in Section 2. In Section 3, which contains precise statements of the main result, we present various dimension-free inequalities for information concentration. We provide a counterexample to a conjecture, which is a natural strengthening of our results. We discuss implications of information concentration in Section 4 with motivating examples. Finally, Section 5 presents the  technical proofs.
\section{Preliminaries}
\subsection{Notations}
For a function $f$, we write $\EE_\mu f \coloneqq \int f d\mu$ and 
$\Var_\mu(f) \coloneqq \int f^2 d\mu - \l( \int f d\mu \r)^2$. We write $X\sim \mu$ for a random variable $X$ associated with the probability measure $\mu$.

In this paper, the norm $\| \cdot \|$ is always the Euclidean norm, and we use $\ip{\cdot}{\cdot}$ for the Euclidean inner product. We use $\nabla V$, $\nabla^2 V$, and $\partial V$ to denote the gradient, Hessian, and subgraident of $V$, respectively. The notation $\mathcal{C}^k$ denotes the class of $k$-times differentiable functions with continuous $k$-th derivatives.

\subsection{Differential Entropy and Log-Concave Distributions}
Let $\mu$ be a probability measure having density $f$ with respect to the Lebesgue measure and let $X \sim \mu$. The differential entropy \citep{cover2012elements} of $X$ is defined as 
\beq \label{eq:diff_entropy}
h(\mu) = h(X) \coloneqq \EE_{\mu} [-\log f(X)].
\eeq The random variable $\tilde{h}(\mu) = \tilde{h}(X) \coloneqq -\log f(X)$ is called the \emph{information content} of $\mu$. 

We study the concentration of information content around the differential entropy:
\beq
\PP \l( | \tilde{h}(X) - h(X)   | > t \r) \leq \alpha(t) \notag
\eeq
where $\alpha: \RR^+ \rightarrow \RR^+$ vanishes rapidly as $t$ increases.

Throughout this paper, we consider \emph{log-concave probability measures}, namely probability measures having density of the form 
\beq
d\mu_V(\x)  = \frac{e^{-V(\x)}}{\int e^{-V}} d\x, \label{eq:log-concave}
\eeq where $V$ is a convex function such that $\int e^{-V} < \infty$.  The function $V$ is called the \emph{potential} of the measure $\mu_V$. For log-concave measures, the concentration of information content is equivalent to the concentration of the potential, i.e.,  $\PP\l( \left| V - \EE_{\mu_V} V \right| > t\r)$.
\section{Dimension Free Concentration of Information for Exp-Concave Potentials}\label{sec:concentration}

This section presents our main results.

We first review the state-of-the-art bound in Section \ref{sec:concentration}.1. In Section \ref{sec:concentration}.2, we demonstrate dimension-free information concentration when the underlying potential $V$ is assumed to be exp-concave. All our results are of sub-exponential type; it is hence natural to ask if the sub-Gaussian counterparts are also true. We show that this is impossible even in dimension 1, by giving a counterexample in Section \ref{sec:concentration}.3. Finally, we highlight some immediate consequences of our main results in Section \ref{sec:concentration}.4. All proofs are deferred to Section \ref{sec:proof}.

\subsection{Previous Art}
The state-of-the-art concentration bound for $V$ is given by \cite{fradelizi2016optimal}:
\begin{theorem}\emph{(Information Concentration for Log-Concave Vectors)}\label{thm:concentration_log-concave}
Let $d\mu_V(\x) \coloneqq \frac{e^{-V(\x)}}{\int e^{-V}} d\x$ be a $d$-dimensional log-concave probability measure. Then, we have
\begin{enumerate}
\item $\Var(V(X)) \leq d$.
\item There exist universal constants $c_1$ and $c_2$ such that\begin{align}
\PP \l( |V - \EE V| > t   \r) \leq c_1 \exp\l(- c_2\min\l( t, \frac{t^2}{d} \r) \r). \label{eq:concentration_log-concave}
\end{align}
\end{enumerate}
\end{theorem}This is the main result of \citep{fradelizi2016optimal} combined with the well-known relation 
\beq 
t - \log(1+t) \simeq \min(t,t^2) \notag
\eeq 
for every $t \geq 0$.

The bound \eqref{eq:concentration_log-concave} matches the tightest known results for $V = \frac{\| \cdot \|^2}{2}$ \citep[i.e., the Gaussian case; see][]{cover1989gaussian}. However, notice that \eqref{eq:concentration_log-concave} has a poor dependence on the dimension $d$, as well as having the exponent being the worst case of $t$ and $\frac{t^2}{d}$.

\subsection{Our Results}
We first recall the definition of exp-concave functions \citep{hazan2016introduction}:
\begin{definition}
A function $V$ is said to be $\eta$-exp-concave if $e^{-\eta V}$ is concave. Equivalently, $V$ is $\eta$-exp-concave if the matrix inequality $\nabla^2 V(\x) \succeq \eta \nabla V(\x) \nabla V(\x)^\top$ holds for all $\x$. Notice that an exp-concave function is necessarily convex.
\end{definition}

We next present three concentration inequalities for $V$ in \textbf{Theorem \ref{thm:concentration}-\ref{thm:separate_concentration}}. \textbf{Theorem \ref{thm:concentration}} serves as the prototype for all the concentration inequalities to come, however with restrictive conditions that severely limit its applicability. To overcome such dilemma, in \textbf{Theorem \ref{thm:bounded_support}} and \textbf{\ref{thm:separate_concentration}} we introduce practically motivated assumptions, and show how the restrictive conditions of \textbf{Theorem \ref{thm:concentration}} can be removed without effecting the concentration.

\subsubsection{Information Concentration: the Strictly Convex Case}
The first main result of this paper is that, for $d\mu_V$ with $V$ being $\eta$-exp-concave and strictly convex, the concentration of information content depends solely on $\eta$.
\begin{theorem}\label{thm:concentration}
Assume that $V \in \C2$ is $\eta$-exp-concave and $\nabla^2 V \succ 0$. Let $d\mu_V(\x) = \frac{e^{-V(\x)}}{\int e^{-V}}d\x$ be the log-concave distribution associated with $V$. Then
\begin{enumerate}
\item $\Var_{\mu_V}(V) \leq \frac{1}{\eta}$.
\item $\PP \Big( | V - \mathbb{E} V  | \geq t \Big) \leq {6}\exp\left({- \max(\sqrt{\eta}, \eta) {t}} \right)$.
\end{enumerate}
\end{theorem}

Notice that when $\frac{1}{\eta} \simeq d$, the bounds in \textbf{Theorem \ref{thm:concentration_log-concave}} and \textbf{\ref{thm:concentration}} yield comparable results. In this sense, $\frac{1}{\eta}$ can be viewed as the ``effective dimension'' regarding the concentration of information content.

\subsubsection{Information Concentration: the General Convex Case}
In many of the applications in learning theory (\textit{cf.}, Section 4), the potential $V$ is not guaranteed to be globally strictly convex. However, we have the following observation:

\def\Scal{{ \mathcal{S} }}
\begin{lemma}\label{lem:spectrum_V}
Assume that $V \in \C2$ is $\eta$-exp-concave. Let $\mathcal{S}^+(\x)$ be the subspace spanned by the eiganvectors corresponding to non-zero eigenvalues of $\nabla^2 V(\x)$. Then $\nabla V(\x) \in \mathcal{S}^+(\x)$ for all $\x$.
\end{lemma}
Simply put, $\nabla^2 V$ may not be strictly convex in all directions, but it is always strictly convex in the direction of $\nabla V$. Our second result shows that in this case, one can drop the global strict convexity of $V$ while retaining exactly the same dimension-free concentration.
\begin{theorem}\label{thm:bounded_support}
Assume that $V \in \C2$ is $\eta$-exp-concave, but not necessarily strictly convex. Let $d\mu_V(\x) = \frac{e^{-V(\x)}}{\int e^{-V}}d\x$ be the log-concave distribution associated with $V$. Then 
\begin{enumerate}
\item $\Var_{\mu_V}(V) \leq \frac{1}{\eta}$.
\item $\PP \Big( | V - \mathbb{E} V  | \geq t \Big) \leq {6}\exp\left({- \max(\sqrt{\eta}, \eta) {t}} \right)$.
\end{enumerate}
\end{theorem}

\subsubsection{Information Concentration in the Presence of Nonsmooth Potential}
The following case appears frequently in machine learning applications: The potential $V$ can be decomposed as $V = V_1 + V_2$, where $V_1$ is a ``nice'' convex function (meaning satisfying either the assumptions in \textbf{Theorem \ref{thm:concentration}} or \textbf{Theorem \ref{thm:bounded_support}}), while $V_2$ is a nonsmooth convex function. Since $V$ is neither differentiable nor strictly convex, results above do not apply. 

Our third result is to show that, in this scenario, the  term $V_1$ in fact enjoys dimension-free concentration as if the nonsmooth term $V_2$ is absent.
\begin{theorem}\label{thm:separate_concentration}
Let $V = V_1 + V_2$, where $V_1$ satisfies the assumptions in either \textbf{Theorem \ref{thm:concentration}} or \textbf{Theorem \ref{thm:bounded_support}}, and $V_2$ is a general convex function. Then we have
\beq
\PP \Big( | V_1 - \mathbb{E} V_1  | \geq t \Big) \leq {6}\exp\left({- \max(\sqrt{\eta}, \eta) {t}} \right),
\eeq
where the probability is with respect to the total measure $d\mu_V$, and $\eta$ is the exp-concavity parameter of $V_1$.
\end{theorem}

\subsection{A Counterexample to Sub-Gaussian Concentration of Information Content}
So far, we have established dimension-free concentration of sub-exponential type under various conditions. An ansatz is whether under the same assumptions, one has dimension-free \emph{sub-Gaussian concentration}; i.e., a deviation inequality of the form 
\beq \label{eq:AppC_subGaussian}
\PP\l( |V - \EE V |  \geq t \r) \leq c_1 e^{-c(\eta)t^2}
\eeq 
for a universal constant $c_1$ and some constant $c(\eta)$ depending only on $\eta$. 

In this subsection, we provide a counterexample to this conjecture, showing that this is impossible even in dimension 1.

Consider the one-dimensional case where $V(x) = -\log x$ and the support is $\Omega \coloneqq (0,1)$. Notice that $V$ is trivially 1-exp-concave. If \eqref{eq:AppC_subGaussian} holds for $V$, then we would have
\begin{align}
\EE e^{\lambda(V-\EE V)^2} &= \int_0^\infty \PP\l( e^{\lambda(V-\EE V)^2} > x \r) dx \notag \\
&\leq  2\lambda \int_0^\infty c_1 e^{-c(\eta)t^2} t e^{\lambda t^2}dt  \notag \\
&< \infty \notag
\end{align}if $\lambda < c(\eta)$. However, a straightforward computation shows that 
\begin{align}
\EE e^{\lambda(V-\EE V)^2} &= \frac{e^{\frac{\lambda}{64}}}{2} \int_0^1 x^{1.25} e^{\lambda(\log x)^2}dx = \infty
\end{align}for every $\lambda >0$. We hence cannot have any sub-Gaussian concentration for $V$.

It is easy to generalize this example to any dimension.


\subsection{Immediate Consequences}
An immediate consequence of information concentration is that many important densities in information theory also concentrate.
\begin{corollary}[Concentration of Information Densities]
Let $f(\x, \y) : \RR^d \times \RR^d \rightarrow \RR$ be a joint log-concave density of the random variable pair $(X,Y)$. Denote the marginal distribution of the first argument by $f(\x) \coloneqq \int_{\RR^d} f(\x,\y)d\y$ and similarly for $f(\y)$, and denote the conditional distribution by $f(\y|\x) \coloneqq \frac{f(\x,\y)}{f(\x)}.$ Then there exist universal constants $c_1, c_2$ such that the following holds:
\begin{enumerate}
\item $\PP \l( \left| -\log f(Y|X) - \EE  [-\log f(Y|X)] \right|> t \r) \leq 2c_1 \exp\l(- \frac{c_2}{2}\min\l( t, \frac{t^2}{d} \r) \r).  $
\item $ \PP \l( \left| -\log \frac{f(X,Y)}{f(X)f(Y)} - \EE  \left[-\log \frac{f(X,Y)}{f(X)f(Y)} \right] \right|> t \r) \leq 3c_1 \exp\l(- \frac{c_2}{3}\min\l( t, \frac{t^2}{d} \r) \r). $
\end{enumerate}
If, in addition, that $-\log f(\x,\y)$, $-\log f(\x)$, and $-\log f(\y)$ are $\eta$-exp-concave and $-\log f(\cdot, \cdot)$ is strictly convex. Then the exponents in the above bounds can be improved to $\max(\sqrt{\eta}, \eta)t$.
\end{corollary}

Notice that $h(Y|X) \coloneqq \EE  \left[-\log f(Y|X) \right]$ is the \emph{conditional (differential) entropy}, and $I(X;Y) \coloneqq \EE  \left[-\log \frac{f(X,Y)}{f(X)f(Y)} \right]$ is the mutual information. The (random) quantities $-\log f(Y|X)$ and $ -\log \frac{f(X,Y)}{f(X)f(Y)}$ play prominent roles in recent advances of non-asymptotic information theory; see \cite{polyanskiy2010channel} and the references therein.\newline

\begin{proof}
A celebrated result of \citet{prekopa1971logarithmic} states that the marginals of log-concave measures are also log-concave. The corollary then follows by the well-known decomposition $h(Y|X) = h(X,Y) - h(X)$ and $I(X;Y) = h(X) + h(Y) - h(X,Y)$.
\end{proof}

\section{Motivating Examples}
Unsurprisingly, information concentration has many applications in learning theory; we present three examples in this section. To avoid lengthy but straightforward calculations, we shall omit the details and refer the readers to proper literature. 

Below, we consider loss functions of the form $L_n(\x) \coloneqq \frac{1}{n}\sum_{i=1}^n \ell_i(\x)$, where $\ell_i$'s are exp-concave. By \textbf{Lemma \ref{lem:properties_exp_concave2}} in Appendix A, the total loss $L_n$ is also exp-concave. Denote the exp-concave parameter of $L_n$ by $\eta$. 



We remark that, in general, $\eta$ can depend on the dimension $d$ or the sample size $n$. A comparison of the favorable regimes for different $\eta$'s is presented in Table 1.

\begin{table}[t]
\begin{center}
  \begin{tabular}[t]{ | l | l | l | l |}
    \hline
    \ & \cite{fradelizi2016optimal} & Ours, $\eta = \Omega(1)$ & Ours, $\eta = \Omega\l(\frac{1}{d}\r)$ \\ \hline
    $t=\Theta(1)$ & $\exp\l(-\frac{1}{d}\r)$ & $\exp\l(-1\r)$ & $\exp\l(-\frac{1}{\sqrt{d}}\r)$ \\ \hline
    $t=\Theta(\sqrt{d})$ & $\exp\l(-1\r)$ & $\exp\l(-\sqrt{d}\r)$ & $\exp\l(-1\r)$ \\ \hline
    $t=\Theta(d)$ & $\exp\l(-d\r)$ & $\exp\l(-d\r)$ & $\exp\l(-\sqrt{d}\r)$ \\
    \hline
  \end{tabular}
  \caption{The deviation $\PP \l(| V - \EE V | > t \r)$ dictated by different concentration inequalities.}
\end{center}
\end{table}

\subsection{High-Probability Regret Bounds for Exponential Weight Algorithms}
Exp-concave losses have received substantial attention in online learning as they exhibit logarithmic regret \citep{hazan2007logarithmic}. One class of algorithms attaining logarithmic regret is based on the \emph{Exponential Weight}, which makes prediction according to
\beq
\x_{t+1} = \EE_{\pi_t} X, \label{eq:prediction}
\eeq
where
\beq
\pi_t(\x) \propto e^{-nL_n(\x)}. \label{eq:exponential_weight}
\eeq
A common belief is that the algorithm \eqref{eq:prediction} is inefficient to implement, and practitioners would more opt into first-order methods such as the \citep[see][]{hazan2007logarithmic} Online Newton Step (which is also somewhat inefficient: every iteration requires inverting a matrix and a projection). However, recent years have witnessed a surge of interest in the sampling schemes, mainly due to its connection to the ultra-simple Stochastic Gradient Descent \citep{welling2011bayesian}. Theoretical \citep{bubeck2015sampling, durmus2016high, dalalyan2017theoretical, dalalyan2017user, cheng2017convergence} and empirical \citep{welling2011bayesian, ahn2012bayesian, rezende2014stochastic, blei2017variational} studies of sampling schemes have now become one of the most active areas in machine learning.

In view of these recent developments, it is natural to consider, instead of the expected prediction \eqref{eq:prediction}, taking samples $X_{t,1}, X_{t,2}, ..., X_{t,N}  \sim \pi_t$ and predict $\bar{X}_t \coloneqq \frac{1}{N}\sum_{i=1}^N X_{t,i}$. The following corollary of our main result establishes the desirable concentration property of $\bar{X}_t$.
\begin{corollary}\label{cor:iid}
Let $\{X_i\}_{i=1}^N$ be i.i.d. samples from the distribution $\frac{e^{-V}}{\int e^{-V}}$. Assume that $V$ satisfies either the assumptions of \textbf{Theorem \ref{thm:concentration}} or \textbf{Theorem \ref{thm:bounded_support}}. Then 
\beq \label{eq:iid}
\PP\l( \left|  \frac{1}{N}\sum_{i=1}^N V(X_i) - \EE V \right|  > t \r)  \leq 2e^{-N(\sqrt{\eta}t-\log 3)}.
\eeq
\end{corollary}
\begin{proof}
For simplicity, assume $\eta = 1$; the general case is similar.

By the classic Chernoff bounding technique, we can compute
\begin{align*}
\PP\l(  \frac{1}{N}\sum_{i=1}^N V(X_i) - \EE V > t \r) &= \PP\l(  e^{ \sum_{i=1}^N \l( V(X_i) - \EE V \r)}   > e^{Nt} \r) \\
&\leq e^{- Nt} \l(\EE e^{\l( V(X) - \EE V \r)}\r)^N \\
&\leq e^{- Nt}\cdot 3^N \\
&= e^{- N(t-\log 3)}, 
\end{align*}where the second inequality follows from \eqref{eq:bound_MGF} with $\eta = 1$.
\end{proof}

Plugging \eqref{eq:iid} into the expected regret bounds for the Exponential Weight algorithm \citep[e.g.,][]{hazan2007logarithmic}, we immediately obtain high-probability regret bounds.

Similar arguments hold for random walk-based approaches in online learning \citep{narayanan2010random}.





\subsection{Posterior Concentration of Bayesian and Pac-Bayesian Analysis}
The (pseudo-)posterior distribution plays a fundamental role in the PAC-Bayesian theory:
\beq \label{eq:posterior}
\hat{\pi} (\x) \propto \exp\l(-  nV_n (\x)\r),
\eeq
where $V_n(\x) = L_n(\x) -  \frac{1}{n} \log\pi_0(\x)$. Here, $\x$ represents the parameter vector and $\pi_0$ is the prior distribution. It is well-known that \eqref{eq:posterior} is optimal in PAC-Bayesian bounds for the expected (over the posterior distribution on the parameter set) population risk \citep{catoni2007pac}. Moreover, when the loss functions $\ell_i$'s are the negative log-likelihood of the data, the optimal PAC-Bayesian posterior \eqref{eq:posterior} coincides with the Bayesian posterior; see \citep{zhang2006information} or the more recent \citep{germain2016pac}.


We now consider the high-probability bound in the following sense: Instead of taking the expectation over $\hat{\pi}$ as previously done, we draw a random sample $X\sim \hat{\pi}$, and ask what is the population risk for $X$. Besides its apparent theoretical interest, such characterization is also important in practice, as there exist many sampling schemes for log-concave distributions $\hat{\pi}$ \citep{lovasz2007geometry, bubeck2015sampling, durmus2016high, dalalyan2017theoretical}, while computing the mean is in general costly (the mean is typically obtained through a large amount of sampling anyway).

A straightforward application of \textbf{Theorem \ref{thm:separate_concentration}} shows that, if the prior $\pi_0$ is log-concave, then $L_n(X)$ concentrates around $\EE_{\hat{\pi}} L_n(X)$; notice that many popular priors (uniform, Gaussian, Laplace, etc.) are log-concave. On the other hand, concentration of the empirical risk $L_n$ around the population risk is a classical theme in statistical learning. To conclude, \textbf{Theorem \ref{thm:separate_concentration}} implies high-probability results for the PAC-Bayesian bounds. In view of the equivalence established in \citep{germain2016pac}, we also obtain concentration for the Bayesian posterior in the case of negative log-likelihood loss.

\subsection{Bayesian Highest Posterior Density Region}
Let $\hat{\pi}$ be the posterior distribution as in \eqref{eq:posterior}. In Bayesian decision theory, the optimal confidence region associated with a level $\alpha$ is given by the Highest Posterior Density (HPD) region \citep{robert2007bayesian}, which is defined as 
\beq \label{eq:HPD}
C^\star_{\alpha} \coloneqq \{ \x \in \RR^d\ | \  V_n(\x) \leq  \gamma_\alpha \}
\eeq
where $\gamma_\alpha$ is chosen so that $\int_{C^\star_{\alpha}}  \hat{\pi}(\x) d\x = 1-\alpha $.

Using concentration of the information content for log-concave distributions, \cite{pereyra2017maximum} showed that $C^\star_{\alpha}$ is contained in the set 
\beq
\tilde{C}_\alpha \coloneqq \{\x \in \RR^d \ | \ V_n(\x) \leq V_n(\x^\star) + dt_\alpha + d  \}, \label{eq:Calpha}
\eeq 
where $\x^\star \coloneqq \argmax_{\x} V_n(\x)$ is the MAP parameter, and $t_\alpha = c_1\sqrt{\frac{\log(1/\alpha)}{d}}$ for some constant $c_1$. A straightforward application of our results shows that, when the data term $L_n$ in $V_n$ is $\eta$-exp-concave, then we can improve \eqref{eq:Calpha}. For simplicity, let us focus on the uniform prior ($\pi_0 = $ constant). Adapting the analysis in \cite{pereyra2017maximum}, we can show that $C^\star_{\alpha}$ is contained in the set
\beq \label{eq:Calphaeta}
\tilde{C}^\eta_\alpha \coloneqq \{\x \in \RR^d \ | \ V_n(\x) \leq V_n(\x^\star) + t^\eta_\alpha + d  \},
\eeq
where $t^\eta_\alpha = c_2\log(1/\alpha) \cdot \sqrt{\frac{n}{{\eta}}}$. Comparing \eqref{eq:Calpha} and \eqref{eq:Calphaeta}, we see that (ignoring logarithmic terms) we get improvements whenever $\eta = \Omega\l(\frac{n}{d}\r)$. This is typically the case in high-dimensional statistics \citep{buhlmann2011statistics} or compressive sensing \citep{ji2008bayesian, foucart2013mathematical} where $n \ll d$. 

Similar results can be established for the Gaussian and Laplace prior, where one can invoke results in \citep{cover1989gaussian} and \citep{talagrand1995concentration} to deduce the concentration of the prior term. We omit the details.
\section{Proofs of the Main Results}\label{sec:proof}
We prove the main results in this section. Our analysis crucially relies on the \emph{variance Brascamp-Lieb inequality}, recalled and elaborated in Section \ref{sec:proof}.1. Section \ref{sec:proof}.2-4 are devoted to the proofs of \textbf{Theorem \ref{thm:concentration}-\ref{thm:separate_concentration}}, respectively. 

\subsection{Proof Ideas}
For a probability measure $\mu$, we say that $\mu$ satisfies the Poincar\'{e} inequality with constant $\lambda_1$ if
\beq
\lambda_1\Var_\mu(f) \leq  \int \| \nabla f \|^2 d\mu \label{eq:poincare}
\eeq for all locally Lipschitz $f$. It is well-known that if \eqref{eq:poincare} is satisfied for $\mu$, then all the Lipschitz functions concentrate exponentially \citep{ledoux2004spectral, ledoux2005concentration}:
\beq
\forall \ 1-\text{Lipschitz } f, \quad  \PP \l( |f-\EE f|>t \r) \leq c_1 e^{-\sqrt{\lambda_1} t} \label{eq:proofs1}
\eeq for some universal constant $c_1$.

At first glance, our theorems seem to have little to do with the Poincar\'{e} inequality, since 
\begin{enumerate}
\item It is not known whether a log-concave distribution satisfies the Poincar\'{e} inequality with a dimension-independent constant \citep[this is the content of the Kannan-Lov\'{a}sz-Simonovits conjecture; see][]{kannan1995isoperimetric,alonso2015approaching}.

\item Typically, the potential $V$ is not Lipschitz (consider the Gaussian distribution where $V(\x) = \frac{\|\x\|^2}{2}$). Moreover, even if $V$ is Lipschitz, the Lipschitz constant often depends on the dimension (consider the exponential distribution where $\| \nabla V \| = \Theta(\sqrt{d})$).
\end{enumerate}

The important observation in this paper is that the appropriate norm in \eqref{eq:poincare} for information concentration is not the Euclidean norm (or any $\ell_p$-norm), but instead the (dual of the) \emph{local norm} defined by the potential $V$ itself, namely $\|\y\|_{\x} \coloneqq \ip{\nabla^2 V(\x)\y}{\y}$. 

\textbf{Lemma \ref{lem:properties_exp_concave1}} in Appendix A expresses the fact that $\eta$-exp-concave functions are Lipschitz with respect to this local norm, and the \emph{Brascamp-Lieb inequality} below provides a suitable strengthening of the Poincar\'{e} inequality:
\begin{theorem}[Brascamp-Lieb Inequality]\label{thm:BL}
Let $d\mu_V(\x) = \frac{e^{-V(\x)}}{\int e^{-V}} d\x$ be a log-concave probability measure with $V \in \C2$ and $\nabla^2 V \succ 0$. Then for all locally Lipschitz function $f \in L_2(\mu_V)$, we have
\beq
\Var_{\mu_V} (f) \leq \int \ip{\nabla^2V^{-1}\nabla f}{\nabla f} d\mu_V. \label{eq:brascamp_lieb}
\eeq
\end{theorem}

We shall see that the Brascamp-Lieb inequality provides precisely the desired control of the Lipschitzness of $V$ in terms of the aforementioned dual local norm. Once this is observed, the rest of the proof is a routine in deducing from Poincar\'{e} inequality the sub-exponential concentration of Lipschitz functions.

We remark that our approach is, in retrospect, completely natural and elementary. However, to the best of our knowledge, our work is the first to combine the Brascamp-Lieb inequality \eqref{eq:brascamp_lieb} with the local norm of the form $\|\y\|_\x \coloneqq \ip{\nabla^2 V(\x)\y}{\y}$.

\subsection{Proof of Theorem \ref{thm:concentration}}
The first assertion is a simple application of the Brascamp-Lieb inequality \eqref{eq:brascamp_lieb} and  \textbf{Lemma \ref{lem:properties_exp_concave1}}.

\ \\
We now prove the concentration inequality. We first show that $\PP \Big( | V - \mathbb{E} V  | \geq t \Big) \leq {6}\exp\left({- \sqrt{\eta} {t}} \right)$. Applying \eqref{eq:brascamp_lieb} to $f = \exp\l({\frac{\lambda (V- \EE V)}{2}}\r)$, we get
\begin{align}
\Var_{\mu_V} (f) &\leq \frac{\lambda^2}{4} \int f^2  \ip{\nabla^2V^{-1} \nabla V}{\nabla V}   d\mu_V \notag \\
&\leq \frac{\lambda^2}{4 \eta} \int f^2     d\mu_V  \label{eq:hold2}
\end{align}by \textbf{Lemma \ref{lem:properties_exp_concave1}}. Let $M(\lambda) \coloneqq \EE \exp\l({{\lambda (V- \EE V)}}\r)$. Then the inequality \eqref{eq:hold2} reads
\begin{equation}
M(\lambda) - M \l(\frac{\lambda}{2} \r)^2 \leq \frac{\lambda^2}{4\eta} M(\lambda),
\end{equation}and hence 
\beq \label{eq:hold3}
M(\lambda) \leq \frac{1}{1-\frac{\lambda^2}{4\eta}}M \l(\frac{\lambda}{2} \r)^2.
\eeq 
Apply \eqref{eq:hold3} recursively to obtain
\begin{align}\label{eq:hold4}
M(\lambda)  & \leq \Pi_{k=1}^{K-1} \l( \frac{1}{1-\frac{\lambda^2}{4^{k+1}\eta}} \r)^{2^k} M\l(\frac{\lambda}{2^K}  \r)^{2^K} .
\end{align}Since $M(\lambda) = 1 + o(\lambda)$, we have $$M\l(\frac{\lambda}{2^K}  \r)^{2^K} = \Big(1 + o\l(\frac{\lambda}{2^K}\r) \Big)^{2^K} \rightarrow 1$$ as $K \rightarrow \infty$. Hence \eqref{eq:hold4} implies
\begin{align}\label{eq:hold5}
M(\lambda)  & \leq \Pi_{k=1}^{\infty} \l( \frac{1}{1-\frac{\lambda^2}{4^{k+1}\eta}} \r)^{2^k},
\end{align}which in turn gives 
\beq
M(\sqrt{\eta}) \leq 3. \label{eq:bound_MGF}
\eeq

The proof can now be completed by the classic Chernoff bounding technique:
\begin{align}
\PP\l( V - \EE V   \geq t \r) &= \PP \l( e^{ \sqrt{\eta} (V - \EE V )}  \geq e^{\sqrt{\eta}t} \r) \notag \\
& \leq e^{-\sqrt{\eta}t} M(\sqrt{\eta}) \notag \\
&\leq 3e^{-\sqrt{\eta}t}. \label{eq:hold6}
\end{align}

Now, the inequality \eqref{eq:hold6} implies that for any 1-exp-concave $V$, we have 
\beq
\PP\l( V - \EE V   \geq t \r) \leq 3 e^{-t}.\notag
\eeq
If $V$ is $\eta$-exp-concave, $\eta V$ is 1-exp-concave, and hence we conclude that 
\beq
\PP\l( V - \EE V   \geq \frac{t}{\eta} \r) \leq 3 e^{-t};  \notag
\eeq
that is to say, 
\beq 
\PP\l( V - \EE V   \geq t \r) \leq 3 e^{-\eta t}. \label{eq:hold7}
\eeq The bound for $\PP\l( V - \EE V   \leq -t \r)$ is similar. 

The proof is completed by taking the best case of \eqref{eq:hold6} and \eqref{eq:hold7}, and applying the union bound.

\subsection{Proof of Lemma \ref{lem:spectrum_V} and Theorem \ref{thm:bounded_support}}
We first prove \textbf{Lemma \ref{lem:spectrum_V}}. 

For any point $\x$, let $\{\a_i\}_{i=1}^k$ be an orthonormal basis for $\Scal^+(\x)$, assumed to have dimension $k$. We extend $\{\a_i\}_{i=1}^k$ to an orthonormal basis for $\RR^d$ as $\{\a_i\}_{i=1}^d$, and we decompose $\nabla V(\x) = \sum_{i=1}^d c_i \a_i$ for some real numbers $c_i$'s. 

\def\nn{{\nonumber}}

For the purpose of contradiction, assume that $\nabla V(\x) \notin \Scal^+(\x)$. Then $c_j \neq 0$ for some $j \in \{ k+1, k+2, ...,d\}$. But then
\beq \nn
\ip{\nabla^2 V(\x) \a_j}{\a_j}   = 0
\eeq
while 
\beq \nn
\a_j^\top \nabla V(\x) \nabla V(\x)^\top \a_j   = c_j^2 > 0,
\eeq
contradicting the exp-concavity of $V$. This finishes the proof of \textbf{Lemma \ref{lem:spectrum_V}}. 

We now turn to \textbf{Theorem \ref{thm:bounded_support}}.

Let $\epsilon >0$ be arbitrarily small, and consider the quantity
\beq \label{eq:thm5_hold1}
\ip{\l( \nabla^2 V + \epsilon I \r)^{-1} \nabla V}{\nabla V}.
\eeq
\textbf{Lemma \ref{lem:spectrum_V}} implies that \eqref{eq:thm5_hold1} is equal to
\beq \label{eq:thm5_hold2}
\ip{\l( \nabla^2 V + \epsilon I_{\Scal^+} \r)^{-1} \nabla V}{\nabla V},
\eeq
where $I_{\Scal^+}$ is the identity map on the subspace $\Scal^+$. Since $V$ is strictly convex restricted to $\Scal^+$, and since $\nabla V \in \Scal^+$, \textbf{Lemma \ref{lem:properties_exp_concave1}} then implies
\beq
 \ip{\l( \nabla^2 V + \epsilon I \r)^{-1} \nabla V}{\nabla V} \leq \frac{1}{\eta} \label{eq:poof_bounded_support_1}
\eeq
for all $\epsilon > 0$.


Consider $\tV = V + \frac{\epsilon\| \cdot\|^2}{2}$, and let $f = \exp\l( \frac{\lambda(V-\EE V)}{2} \r)$. Since $\tV$ is strictly convex, we may invoke the Brascamp-Lieb inequality \eqref{eq:brascamp_lieb} to conclude
\begin{align}
\Var_{\mu_\tV} (f) &\leq  \frac{\lambda^2}{4} \int f^2\ip{\nabla^2  \tV^{-1} \nabla V}{\nabla V} d\mu_\tV \notag \\
&\leq  \frac{\lambda^2}{4\eta} \int f^2 d\mu_\tV  \label{eq:poof_bounded_support_2}
\end{align}where the second inequality follows from \eqref{eq:poof_bounded_support_1}. Letting $\epsilon \rightarrow 0$ in \eqref{eq:poof_bounded_support_2} then gives 
\beq
\Var_{\mu_V}(f) \leq  \frac{\lambda^2}{4\eta} \int f^2 d\mu_V.
\eeq

The rest of the proof is similar to that of \textbf{Theorem \ref{thm:concentration}}; we omit the details.

\subsection{Proof of Theorem \ref{thm:separate_concentration}}
We will need the following strengthened Brascamp-Lieb inequality, which might be of independent interest. Once the following theorem is established, one can follow a similar proof as in Section \ref{sec:proof}.2. We omit the details, and focus on the proof of the following theorem in the rest of this subsection.
\begin{theorem}[Nonsmooth Brascamp-Lieb Inequality]\label{thm:BL_nonsmooth}
Let $d\mu_{\tilde{V}}(\x) = \frac{e^{-\tilde{V}(\x)}}{\int e^{-\tilde{V}}} d\x$ be a log-concave measure with $\tilde{V} = V + U$, where $V \in \C2$, $\nabla^2 V \succ 0$, and $U$ is convex but possibly non-differentiable. Then for all locally Lipschitz function $f \in L_2(\mu_{\tV})$, we have
\beq
\Var_{\mu_{\tilde{V}}} (f) \leq \int \ip{\nabla^2V^{-1}\nabla f}{\nabla f} d\mu_{\tilde{V}}. \label{eq:brascamp_lieb_nonsmooth}
\eeq
\end{theorem}

\begin{proof}[Proof of \textbf{Theorem \ref{thm:BL_nonsmooth}}]
Define the cost functions
\beq
c_{\tV}(\x,\y) \coloneqq \tilde{V}(\y) - \tV(\x) - \ip{\partial \tV(\x)}{\y-\x} \label{eq:nonsmooth_cost}
\eeq and
\beq
c_V(\x,\y) \coloneqq V(\y) - V(\x) - \ip{\nabla V(\x)}{\y-\x}. \label{eq:smooth_cost}
\eeq
By \textbf{Proposition 1.1} of \cite{cordero2017transport} (see also p.482 for the non-differentiable case), we know that the measure $d\mu_{\tilde{V}}$ satisfies the transportation cost inequality:
\beq
\calW_{c_{\tV}}(\mu_{\tV},\nu) \leq D(\nu \| \mu_{\tV}) \notag
\eeq
for any probability measure $\nu$. Here, $\calW_{c_{\tV}}(\mu_{\tV}, \nu) \coloneqq \inf_{X,Y} \EE c_{\tV}(X,Y)$ where the infimum is over all joint distributions with marginals $X \sim \mu_{\tV}$ and $Y \sim \nu$, and $D(\nu \| \mu_{\tV})$ is the relative entropy between $\nu$ and $\mu_{\tV}$. Since $U$ is convex, we have $c_{\tV}(\x,\y) \geq c_{V}(\x,\y)$ for all $\x,\y$, and hence $\mu_{\tV}$ satisfies the weaker transportation cost inequality 
\beq
\calW_{c_{V}}(\mu_{\tV},\nu) \leq D(\nu \| \mu_{\tV}). \label{eq:transportation_cost_weak}
\eeq
The theorem can then be deduced from a standard linearization procedure that is well-known since the classic \citep{otto2000generalization}. The rest of the proof below is a suitable adaptation of the version in \citep{cordero2017transport}.

Since continuous functions with compact support are dense in $L_2(\mu_{\tilde{V}})$, we will prove \textbf{Theorem \ref{thm:BL_nonsmooth}} for any continuous function with compact support. Notice that such functions are necessarily Lipschitz and hence differentiable $\mu_{\tilde{V}}$-almost everywhere. Since modifying $f$ in a set of $\mu_{\tilde{V}}$-measure 0 does not effect \eqref{eq:brascamp_lieb_nonsmooth}, we may henceforth assume that $f \in \mathcal{C}^1$ and has compact support. 

\def\h{{ \mathbf{h}  }}

Since $V \in \C2$, $\nabla^2 V(\y)$ is uniformly continuous on any compact set, and hence we have 
\beq \label{eq:hold12}
c_V(\y+\h,\y) = \frac{1}{2} \ip{\nabla^2 V(\y)\h}{\h} + \| \h\|^2\cdot o(1)
\eeq
uniformly in $\y$ on any compact set when $\h\rightarrow 0$. Assume for the moment that $c_V(\x,\y) \geq \frac{\delta}{2}\|\x-\y\|^2$ for some $\delta >0$. Given any function $g \in \mathcal{C}^1$ with compact support and $\int g d\mu_{\tV} = 0$, introduce the infimal convolution associated with the cost $c_V$:
\beq \label{eq:inf_conv}
Q_c(g)(\y) \coloneqq \inf_{\x} \{ g(\x) +  c_V(\x,\y)  \},
\eeq
whence $Q_c(g)(\y) - g(\x) \leq c_V(\x,\y).$ By the definition of $\calW_{c_{V}}$, for any joint probability measure $\pi$ having marginals $\mu_{\tV}$ and $\nu$, we must have
\beq
\calW_{c_{V}}(\mu_{\tV}, \nu) =  \inf_\pi \int c_V(\x,\y) d\pi(\x,\y) \geq \int Q_c(g) d\nu - \int g d\mu_{\tV}.
\eeq
Consider the infimum convolution of $\epsilon g$:
\beq  \label{eq:hold14}
Q_c(\epsilon g)(\y) = \inf_x \{ \epsilon g(\x) + c_V(\x,\y) \} = \inf_{\h} \{\epsilon g(\y+\h) + c_V(\y+\h,\y)  \}.
\eeq
Let $\h_{\epsilon} = \h_{\epsilon}(\y)$ denote a point where the infimum is achieved. Since $g$ is globally Lipschitz, say with constant $L$, 
\begin{align}\label{eq:hold13}
\epsilon g(\y+\h_{\epsilon}) + c_{V}(\y+\h_{\epsilon},\y) &\geq \epsilon g(\y) - \epsilon L \| \h_\epsilon\| + \delta\| \h_{\epsilon} \|^2.
\end{align}On the other hand, by setting $\h = 0$ in \eqref{eq:hold14}, we see that $\epsilon g(\y+\h_{\epsilon}) + c_{V}(\y+\h_{\epsilon},\y) \leq \epsilon g(\y)$. Combining this with \eqref{eq:hold13} gives
\beq \label{eq:h_is_order_eps}
\| \h_\epsilon \| \leq \frac{L}{\delta} \epsilon.
\eeq
Notice that \eqref{eq:h_is_order_eps} does not depend on $y$, and hence $\| \h_\epsilon \| = O(\epsilon)$ uniformly in $\y$. 

As $g$ is compactly supported, we have $\sup |g| = M <\infty$. Let $\Omega$ be the support of $g$, and let $B_\epsilon \coloneqq \{ \x \ | \ \exists \y \in \Omega,   \| \x - \y \|^2 \leq \frac{2\epsilon M}{\delta} \}.$ We claim that $Q_c(\epsilon g) \geq 0$ on $B_\epsilon^c$. Indeed, for any $\y \in B_\epsilon^c$, suppose that the infimum of $\x$ in \eqref{eq:hold14} is attained in $\Omega$. Then $Q_c(\epsilon g) \geq    \frac{\delta}{2} \cdot \frac{2\epsilon M}{\delta}  -   \epsilon M = 0$, since $c_V(\x,\y) \geq \frac{\delta}{2}\|\x-\y\|^2$ and $\inf_{\x\in \Omega}\|\y-\x\|^2\geq \frac{2\epsilon M}{\delta}$. On the other hand, if the infimum of $\x$ in \eqref{eq:hold14} is attained outside $\Omega$, then $\epsilon g = 0$ and $c_V \geq 0$ implies that  $Q_c(\epsilon g) \geq 0$.

For the sake of linearization, set $d\nu = \l(1+\epsilon f \r)d\mu_{\tV}$ for some $f \in \mathcal{C}^1$ with $\int f d\mu_{\tV} = 0$. We now compute
\begin{align}
\calW_{c_{V}}(\mu_{\tV}, (1+\epsilon f )d\mu_{\tV}) &\geq \int Q_c(\epsilon g) \l(1+\epsilon f\r)d\mu_{\tV}  &\textup{since } \int g d\mu_{\tV} = 0, \notag\\
&\geq \int_{B_\epsilon} Q_c(\epsilon g) \l(1+\epsilon f\r)d\mu_{\tV}  &\textup{since } Q_c(\epsilon g)  \geq 0 \textup{ on } B_\epsilon^c. \label{eq:W_versus_Qc}
\end{align}As the set $B_\epsilon$ is itself compact, we have, uniformly in $\y$,
\begin{align}
 Q_c(\epsilon g)(\y) &= \epsilon g(\y+\h_\epsilon) + c(\y+\h_\epsilon, \y) \notag \\
&= \epsilon g(\y) + \epsilon \ip{\nabla  g(\y)}{\h_\epsilon}  + \frac{1}{2} \ip{\nabla^2 V(\y) \h_\epsilon }{\h_\epsilon}  + o(\epsilon^2)  \label{eq:hold15}
\end{align}where the last line follows by \eqref{eq:h_is_order_eps}. Noticing that \eqref{eq:hold15} is convex in $\h_\epsilon$, we can find its minimum (up to $o(\epsilon^2)$) and write
\beq \label{eq:hold16}
 Q_c(\epsilon g)(\y)  \geq \epsilon g(\y) - \frac{\epsilon^2}{2}\ip{\nabla^2 V^{-1}(\y) \nabla  g(\y)}{\nabla  g(\y)}  + o(\epsilon^2).
\eeq
Multiplying \eqref{eq:hold16} by $1+\epsilon f$ and integrate on $B_\epsilon$ w.r.t. $d\mu_{\tV}$, we get, using \eqref{eq:W_versus_Qc} and $\int f d\mu_{\tV} = \int g d\mu_{\tV} = 0$,
\begin{align}
\frac{1}{\epsilon^2} \calW_{c_{V}} \l(\mu_{\tV}, \l(1+\epsilon f \r) d\mu_{\tV} \r) \geq \int_{B_{\epsilon}} fgd\mu_{\tV} - \frac{1}{2} \int_{B_{\epsilon}} \ip{\nabla^2V^{-1}\nabla g}{\nabla g} d\mu_{\tV} + o(1). \label{eq:hold17}
\end{align}By definition, $B_\epsilon$ contains the support of $g$, and hence the integrals in \eqref{eq:hold17} are in fact over the whole space. We hence conclude
\beq \label{eq:hold18}
\liminf_{\epsilon \rightarrow 0} \frac{1}{\epsilon^2} \calW_{c_{V}} \l(\mu_{\tV}, \l(1+\epsilon f \r) d\mu_{\tV} \r) \geq \int_{} fgd\mu_{\tV} - \frac{1}{2} \int_{} \ip{\nabla^2V^{-1}\nabla g}{\nabla g} d\mu_{\tV}.
\eeq
Replacing $g$ by $\lambda g$ in \eqref{eq:hold18} and optimizing over $\lambda$, we get
\beq \label{eq:hold20}
\frac{ \l(\int fg d\mu_{\tV} \r)^2}{2\int \ip{\nabla V^{-1}\nabla g}{\nabla g}  d\mu_{\tV} } \leq \liminf_{\epsilon \rightarrow 0} \frac{1}{\epsilon^2} \calW_{c_{V}} \l(\mu_{\tV}, \l(1+\epsilon f \r) d\mu_{\tV} \r).  
\eeq

Moreover, using $\log(1+x) = x -\frac{x^2}{2}+ o(x^2)$ and $\int f d\mu_{\tV} = 0$, we can compute
\begin{align}
D\l(   \l(1+\epsilon f\r)d\mu_{\tV}  \|  \mu_{\tV} \r) &= \int \l(1+\epsilon f \r) \l(\epsilon f - \frac{\epsilon^2 f^2}{2}  + o(\epsilon^2) \r) d\mu_{\tV} \notag\\
&= \frac{\epsilon^2}{2} \int f^2 d\mu_{\tV} + o(\epsilon^2). \label{eq:hold19}
\end{align}

In the case of $c_V(\x,\y) \geq \frac{\delta}{2}\|\x-\y\|^2$, \textbf{Theorem \ref{thm:BL_nonsmooth}} then follows by using $g=f$ in \eqref{eq:hold20} and combing \eqref{eq:transportation_cost_weak} and \eqref{eq:hold19}. For general case, replace $V$ by $V+\frac{\delta\| \cdot \|^2}{2}$ and take $\delta \rightarrow 0$ in \eqref{eq:hold18} and \eqref{eq:hold19} to deduce the same inequalities.

\end{proof}

\section{Conclusion}
We have shown that for log-concave distributions with exp-concave potentials, the information concentration is dictated by its exp-concavity parameter $\eta$. Information theoretically speaking, $\eta$ (or rather $\frac{1}{\eta}$) can be viewed as some sort of effective dimension, in the sense that $\frac{1}{\eta}$ and $d$ play very similar roles in both the variance and concentration controls, the former for log-concave measures with exp-concave potential and the latter for general log-concave measures. Such a understanding enables us to derive high-probability results for many of the machine learning algorithms, including the Bayesian, PAC-Bayesian, and Exponential Weight type approaches.





\acks{This project has received funding from the European Research Council (ERC) under the European Union’s Horizon 2020 research and innovation programme (grant agreement n$^o$ 725594 - time-data).}



\appendix
\section*{Appendix A. Properties of Exp-Concave Functions}

\label{app:properties_exp_concave}


We present two useful properties of exp-concave functions in this appendix. While these properties are well-known to the experts, we provide the proofs for completeness.

\begin{lemma}\label{lem:properties_exp_concave1}
Assume that $V$ is $\eta$-exp-concave and $\nabla^2 V \succ 0$. Then we have
\beq
  \ip{\nabla^2V^{-1}(\x) \nabla V(\x)}{\nabla V(\x)} \leq \frac{1}{\eta} \label{eq:properties_exp_concave1}
\eeq
for all $\x$.
\end{lemma}
\begin{proof}
Since $V$ is $\eta$-exp-concave, we have 

\beq
\frac{1}{\eta}\nabla^2 V \succeq  \nabla V \nabla V^\top.
\eeq
Let $v = \frac{\nabla V}{\|\nabla V\|}$ and $R \coloneqq I - vv^\top$. For any $\delta >0$, we have 
\begin{align}
\frac{1}{\eta}\nabla^2 V + \delta I &\succeq \frac{1}{\eta}\nabla^2 V + \delta R  \notag\\
&\succeq \nabla V \nabla V^\top + \delta R \notag\\
&= \l(  \| \nabla V\|^2 -\delta  \r)vv^\top + \delta I \notag\\
& \succ 0
\end{align}if $\delta < \| \nabla V\|^2$. Using the fact that $ B \succeq A \succ 0$ implies $A^{-1} \succeq B^{-1} \succ 0$, we get
\beq
\l(\frac{1}{\eta}\nabla^2 V + \delta I  \r)^{-1} \preceq \frac{1}{\delta} \l( I + t vv^\top\r)^{-1}
\eeq
where $t \coloneqq \frac{\| \nabla V\|^2}{\delta} -1$. The Sherman–Morrison formula implies \beq \l( I + t vv^\top\r)^{-1} = I - \frac{t}{1+t}vv^\top,\eeq and hence, 
\begin{align}
\ip{\l(\frac{1}{\eta}\nabla^2 V + \delta I  \r)^{-1} \nabla V}{\nabla V} &\leq \frac{1}{\delta} \l( \|\nabla V\|^2 - \frac{t}{1+t} \| \nabla V\|^2 \r)\notag \\
&= \frac{1+t}{\| \nabla V \|^2} \l( \|\nabla V\|^2 - \frac{t}{1+t} \| \nabla V\|^2 \r) \notag\\
&= 1. \label{eq:approx<=1}
\end{align}On the other hand, since $\nabla^2V \succ 0$, we have 
\begin{align} 
\lim_{\delta \rightarrow 0} \ip{\l(\frac{1}{\eta}\nabla^2 V + \delta I  \r)^{-1} \nabla V}{\nabla V}  &= \ip{\l(\frac{1}{\eta}\nabla^2 V  \r)^{-1} \nabla V}{\nabla V} \notag\\
&= \eta \ip{ \nabla^2 V ^{-1} \nabla V}{\nabla V}.  \notag 
\end{align}The proof is hence completed by letting $\delta \rightarrow 0$ in \eqref{eq:approx<=1}.
\end{proof}

\begin{lemma}\label{lem:properties_exp_concave2}
Let $V_i$'s be $\eta_i$-exp-concave functions for $i = 1, 2, ..., k$. Then $\sum_{i=1}^k V_i$ is $\eta$-exp-concave with $\frac{1}{\eta} = \sum_{i=1}^k \frac{1}{\eta_i}$.
\end{lemma}
\begin{proof}
Let $X$ be any random variable. Using the exp-concavity and H\"{o}lder's inequality, we get 
\begin{align*}
e^{-\eta (V_1 + V_2) (\EE X)} &\geq  \l(\EE e^{-\eta_1 V_1(X)} \r)^{\frac{\eta}{\eta_1}} \cdot \l(\EE e^{-\eta_2 V_2(X)} \r)^{\frac{\eta}{\eta_2}} \\
&=  \|e^{-V_1}\|_{\eta_1}^\eta \cdot \|e^{-V_2}\|_{\eta_2}^\eta \\
&\geq \| e^{-(V_1 +  V_2)}\|_{\eta}^\eta \\
&= \EE e^{-\eta \big(V_1(X) + V_2(X) \big)}.
\end{align*}Here, $\|e^{-V_1}\|_{\eta_1} \coloneqq \l( \EE e^{-\eta_1V_1} \r)^{\frac{1}{\eta_1}}$ and similarly for $\|e^{-V_2}\|_{\eta_2}$.

The general case follows from induction.
\end{proof}

\vskip 0.2in
\bibliography{biblio}

\begin{thebibliography}{37}
\providecommand{\natexlab}[1]{#1}
\providecommand{\url}[1]{\texttt{#1}}
\expandafter\ifx\csname urlstyle\endcsname\relax
  \providecommand{\doi}[1]{doi: #1}\else
  \providecommand{\doi}{doi: \begingroup \urlstyle{rm}\Url}\fi

\bibitem[Ahn et~al.(2012)Ahn, Korattikara, and Welling]{ahn2012bayesian}
Sungjin Ahn, Anoop Korattikara, and Max Welling.
\newblock Bayesian posterior sampling via stochastic gradient fisher scoring.
\newblock In \emph{Proceedings of the 29th International Coference on
  International Conference on Machine Learning}, pages 1771--1778, 2012.

\bibitem[Alonso-Guti{\'e}rrez and Bastero(2015)]{alonso2015approaching}
David Alonso-Guti{\'e}rrez and Jes{\'u}s Bastero.
\newblock \emph{Approaching the Kannan-Lov{\'a}sz-Simonovits and variance
  conjectures}, volume 2131.
\newblock Springer, 2015.

\bibitem[Blei et~al.(2017)Blei, Kucukelbir, and McAuliffe]{blei2017variational}
David~M Blei, Alp Kucukelbir, and Jon~D McAuliffe.
\newblock Variational inference: A review for statisticians.
\newblock \emph{Journal of the American Statistical Association}, \penalty0
  (just-accepted), 2017.

\bibitem[Bobkov et~al.(2011)Bobkov, Madiman, et~al.]{bobkov2011concentration}
Sergey Bobkov, Mokshay Madiman, et~al.
\newblock Concentration of the information in data with log-concave
  distributions.
\newblock \emph{The Annals of Probability}, 39\penalty0 (4):\penalty0
  1528--1543, 2011.

\bibitem[Bubeck et~al.(2015)Bubeck, Eldan, and Lehec]{bubeck2015sampling}
S{\'e}bastien Bubeck, Ronen Eldan, and Joseph Lehec.
\newblock Sampling from a log-concave distribution with projected langevin
  monte carlo.
\newblock \emph{arXiv preprint arXiv:1507.02564}, 2015.

\bibitem[B{\"u}hlmann and Van De~Geer(2011)]{buhlmann2011statistics}
Peter B{\"u}hlmann and Sara Van De~Geer.
\newblock \emph{Statistics for high-dimensional data: methods, theory and
  applications}.
\newblock Springer Science \& Business Media, 2011.

\bibitem[Catoni(2007)]{catoni2007pac}
Olivier Catoni.
\newblock \emph{PAC-Bayesian Supervised Classification: the Thermodynamics of
  Statistical Learning}.
\newblock Institute of Mathematical Statistics, 2007.

\bibitem[Cesa-Bianchi et~al.(2006)Cesa-Bianchi, Lugosi, and
  Prediction]{cesa2006games}
Nicolo Cesa-Bianchi, Gabor Lugosi, and Learning Prediction.
\newblock Games, 2006.

\bibitem[Cheng and Bartlett(2018)]{cheng2017convergence}
Xiang Cheng and Peter Bartlett.
\newblock Convergence of langevin mcmc in kl-divergence.
\newblock \emph{Proceedings of Machine Learning Research}, 83, 2018.

\bibitem[Cordero-Erausquin(2017)]{cordero2017transport}
Dario Cordero-Erausquin.
\newblock Transport inequalities for log-concave measures, quantitative forms
  and applications.
\newblock \emph{Canadian Journal of Mathematics}, 69:\penalty0 481--501, 2017.

\bibitem[Cover and Pombra(1989)]{cover1989gaussian}
Thomas~M Cover and Sandeep Pombra.
\newblock Gaussian feedback capacity.
\newblock \emph{IEEE Transactions on Information Theory}, 35\penalty0
  (1):\penalty0 37--43, 1989.

\bibitem[Cover and Thomas(2012)]{cover2012elements}
Thomas~M Cover and Joy~A Thomas.
\newblock \emph{Elements of information theory}.
\newblock John Wiley \& Sons, 2012.

\bibitem[Dalalyan(2017)]{dalalyan2017theoretical}
Arnak~S Dalalyan.
\newblock Theoretical guarantees for approximate sampling from smooth and
  log-concave densities.
\newblock \emph{Journal of the Royal Statistical Society: Series B (Statistical
  Methodology)}, 79\penalty0 (3):\penalty0 651--676, 2017.

\bibitem[Dalalyan and Karagulyan(2017)]{dalalyan2017user}
Arnak~S Dalalyan and Avetik~G Karagulyan.
\newblock User-friendly guarantees for the langevin monte carlo with inaccurate
  gradient.
\newblock \emph{arXiv preprint arXiv:1710.00095}, 2017.

\bibitem[Dalalyan et~al.(2016)Dalalyan, Grappin, and
  Paris]{dalalyan2016exponentially}
Arnak~S Dalalyan, Edwin Grappin, and Quentin Paris.
\newblock On the exponentially weighted aggregate with the laplace prior.
\newblock \emph{arXiv preprint arXiv:1611.08483}, 2016.

\bibitem[Durmus and Moulines(2016)]{durmus2016high}
Alain Durmus and Eric Moulines.
\newblock High-dimensional bayesian inference via the unadjusted langevin
  algorithm.
\newblock 2016.

\bibitem[Foucart and Rauhut(2013)]{foucart2013mathematical}
Simon Foucart and Holger Rauhut.
\newblock \emph{A mathematical introduction to compressive sensing}, volume~1.
\newblock Birkh{\"a}user Basel, 2013.

\bibitem[Fradelizi et~al.(2016)Fradelizi, Madiman, and
  Wang]{fradelizi2016optimal}
Matthieu Fradelizi, Mokshay Madiman, and Liyao Wang.
\newblock Optimal concentration of information content for log-concave
  densities.
\newblock In \emph{High Dimensional Probability VII}, pages 45--60. Springer,
  2016.

\bibitem[Germain et~al.(2016)Germain, Bach, Lacoste, and
  Lacoste-Julien]{germain2016pac}
Pascal Germain, Francis Bach, Alexandre Lacoste, and Simon Lacoste-Julien.
\newblock Pac-bayesian theory meets bayesian inference.
\newblock In \emph{Advances in Neural Information Processing Systems}, pages
  1884--1892, 2016.

\bibitem[Hazan(2016)]{hazan2016introduction}
Elad Hazan.
\newblock Introduction to online convex optimization.
\newblock \emph{Foundations and Trends{\textregistered} in Optimization},
  2\penalty0 (3-4):\penalty0 157--325, 2016.

\bibitem[Hazan et~al.(2007)Hazan, Agarwal, and Kale]{hazan2007logarithmic}
Elad Hazan, Amit Agarwal, and Satyen Kale.
\newblock Logarithmic regret algorithms for online convex optimization.
\newblock \emph{Machine Learning}, 69\penalty0 (2):\penalty0 169--192, 2007.

\bibitem[Ji et~al.(2008)Ji, Xue, and Carin]{ji2008bayesian}
Shihao Ji, Ya~Xue, and Lawrence Carin.
\newblock Bayesian compressive sensing.
\newblock \emph{IEEE Transactions on Signal Processing}, 56\penalty0
  (6):\penalty0 2346--2356, 2008.

\bibitem[Kannan et~al.(1995)Kannan, Lov{\'a}sz, and
  Simonovits]{kannan1995isoperimetric}
Ravi Kannan, L{\'a}szl{\'o} Lov{\'a}sz, and Mikl{\'o}s Simonovits.
\newblock Isoperimetric problems for convex bodies and a localization lemma.
\newblock \emph{Discrete \& Computational Geometry}, 13\penalty0 (1):\penalty0
  541--559, 1995.

\bibitem[Ledoux(2004)]{ledoux2004spectral}
Michel Ledoux.
\newblock Spectral gap, logarithmic sobolev constant, and geometric bounds.
\newblock \emph{Surveys in differential geometry}, 9:\penalty0 219--240, 2004.

\bibitem[Ledoux(2005)]{ledoux2005concentration}
Michel Ledoux.
\newblock \emph{The concentration of measure phenomenon}.
\newblock American Mathematical Soc., 2005.

\bibitem[Lov{\'a}sz and Vempala(2007)]{lovasz2007geometry}
L{\'a}szl{\'o} Lov{\'a}sz and Santosh Vempala.
\newblock The geometry of logconcave functions and sampling algorithms.
\newblock \emph{Random Structures \& Algorithms}, 30\penalty0 (3):\penalty0
  307--358, 2007.

\bibitem[Narayanan and Rakhlin(2010)]{narayanan2010random}
Hariharan Narayanan and Alexander Rakhlin.
\newblock Random walk approach to regret minimization.
\newblock In \emph{Advances in Neural Information Processing Systems}, pages
  1777--1785, 2010.

\bibitem[Otto and Villani(2000)]{otto2000generalization}
Felix Otto and C{\'e}dric Villani.
\newblock Generalization of an inequality by talagrand and links with the
  logarithmic sobolev inequality.
\newblock \emph{Journal of Functional Analysis}, 173\penalty0 (2):\penalty0
  361--400, 2000.

\bibitem[Pereyra(2017)]{pereyra2017maximum}
Marcelo Pereyra.
\newblock Maximum-a-posteriori estimation with bayesian confidence regions.
\newblock \emph{SIAM Journal on Imaging Sciences}, 10\penalty0 (1):\penalty0
  285--302, 2017.

\bibitem[Polyanskiy(2010)]{polyanskiy2010channel}
Yury Polyanskiy.
\newblock \emph{Channel coding: non-asymptotic fundamental limits}.
\newblock Princeton University, 2010.

\bibitem[Pr\'{e}kopa(1971)]{prekopa1971logarithmic}
Andr\'{a}s Pr\'{e}kopa.
\newblock Logarithmic concave measures with applications.
\newblock \emph{Acta Sci. Math}, 32:\penalty0 301--316, 1971.

\bibitem[Raginsky et~al.(2013)Raginsky, Sason,
  et~al.]{raginsky2013concentration}
Maxim Raginsky, Igal Sason, et~al.
\newblock Concentration of measure inequalities in information theory,
  communications, and coding.
\newblock \emph{Foundations and Trends{\textregistered} in Communications and
  Information Theory}, 10\penalty0 (1-2):\penalty0 1--246, 2013.

\bibitem[Rezende et~al.(2014)Rezende, Mohamed, and
  Wierstra]{rezende2014stochastic}
Danilo~Jimenez Rezende, Shakir Mohamed, and Daan Wierstra.
\newblock Stochastic backpropagation and approximate inference in deep
  generative models.
\newblock \emph{arXiv preprint arXiv:1401.4082}, 2014.

\bibitem[Robert(2007)]{robert2007bayesian}
Christian Robert.
\newblock \emph{The Bayesian choice: from decision-theoretic foundations to
  computational implementation}.
\newblock Springer Science \& Business Media, 2007.

\bibitem[Talagrand(1995)]{talagrand1995concentration}
Michel Talagrand.
\newblock Concentration of measure and isoperimetric inequalities in product
  spaces.
\newblock \emph{Publications Mathematiques de l'IHES}, 81\penalty0
  (1):\penalty0 73--205, 1995.

\bibitem[Welling and Teh(2011)]{welling2011bayesian}
Max Welling and Yee~W Teh.
\newblock Bayesian learning via stochastic gradient langevin dynamics.
\newblock In \emph{Proceedings of the 28th International Conference on Machine
  Learning (ICML-11)}, pages 681--688, 2011.

\bibitem[Zhang(2006)]{zhang2006information}
Tong Zhang.
\newblock Information-theoretic upper and lower bounds for statistical
  estimation.
\newblock \emph{IEEE Transactions on Information Theory}, 52\penalty0
  (4):\penalty0 1307--1321, 2006.

\end{thebibliography}
\end{document}